\documentclass[letterpaper]{article} 
\usepackage[preprint]{aaai2027}  

\usepackage[table]{xcolor}   
\definecolor{hlgreen}{RGB}{76,175,80}
\newcommand{\hla}{\cellcolor{hlgreen!45}}  
\newcommand{\hlb}{\cellcolor{hlgreen!26}}  
\newcommand{\hlc}{\cellcolor{hlgreen!13}}  

\usepackage[hyphens]{url}  
\usepackage{graphicx} 
\usepackage{natbib}  
\usepackage{caption} 
\usepackage{algorithm}
\usepackage{algorithmic}
\usepackage{newfloat}
\usepackage{listings}
\DeclareCaptionStyle{ruled}{labelfont=normalfont,labelsep=colon,strut=off} 
\floatstyle{ruled}
\newfloat{listing}{tb}{lst}{}
\floatname{listing}{Listing}

\usepackage{booktabs}
\usepackage{amsmath}
\usepackage{amssymb}   
\usepackage{amsthm}    
\newtheorem{proposition}{Proposition}
\usepackage{multirow}  
\usepackage{xcolor}
\usepackage{xspace}

\newcommand{\methodname}{AdaMem\xspace}

\usepackage{comment}

\title{\methodname: Adaptive Memory Token Allocation for Soft Compression in Retrieval-Augmented Generation}
\author{
    Artem Sakhno,
    Grigorii Davydenko,
    Omar Zoloev,\\
    Julia Belikova,
    Andrey Savchenko,
    Maksim Makarenko\thanks{Correspondence to: sakno.ad18@physics.msu.ru, maksim.makarenko.research@gmail.com}
}
\affiliations{}

\begin{document}

\maketitle

\begin{abstract}
Retrieval-augmented generation (RAG) improves language models with retrieved evidence, but processing many long passages is costly and can introduce distracting information. Soft compression addresses this challenge by encoding passages as compact sequences of continuous memory embeddings before generation. However, existing methods typically assign each retained passage an identical number of memory embeddings, irrespective of its query-specific relevance. To address this, we propose AdaMem, a relevance-guided soft-compression framework that maps learned passage-relevance estimates to a query-dependent allocation of a fixed memory-token budget. A shared query-conditioned compressor produces both continuous passage memories and relevance scores in a single pass; a deterministic allocation rule assigns more memory tokens to higher-scoring passages and can omit low-scoring ones. Across six open-domain QA benchmarks, AdaMem consistently outperforms OSCAR (the closely matched soft-compression baseline that uses uniform allocation) as well as other soft-compression methods at matched memory budgets.  
Under standard 16$\times$ compression, AdaMem improves sub-string match by up to 3.2 points (5.5\%) over uniform allocation baseline, with an average relative gain of 3.4\%; under aggressive 64$\times$ compression the average relative gain grows to 14.6\%, with a maximum of 9.8 points (19.7\%) on PopQA. AdaMem matches the answer quality of the uncompressed at up to 4$\times$ lower inference latency than full context baseline. AdaMem retains an efficiency profile comparable
to the uniform-compression baseline, while achieving up to $4\times$ lower inference latency than full-context inference.
Thus, relevance-guided memory allocation is particularly effective when retrieval pools are large and the available memory budget is tight.
\end{abstract}


\section{Introduction}

Retrieval-augmented generation (RAG) augments language models with external knowledge retrieved from large document collections, improving performance on knowledge-intensive tasks and grounding generation in retrieved evidence \cite{lewis2020retrieval}. However, processing many retrieved documents is computationally expensive~\citep{xu2024recomp,louis2026oscar}. Moreover, irrelevant passages can distract the generator and degrade answer quality~\citep{amiraz-etal-2025-distracting}. Increasing the context window does not fully resolve this issue, as language models may use information in long multi-document inputs unevenly~\citep{liu-etal-2024-lost}. Even when total context length is fixed, increasing the number of retrieved documents can reduce QA performance, suggesting that multi-document processing is a distinct challenge from long-context handling~\citep{levy-etal-2025-documents}. Efficiently selecting and representing useful retrieved evidence is therefore central to scalable and reliable RAG.

\begin{figure}[t]
    \centering
    \includegraphics[width=\columnwidth]{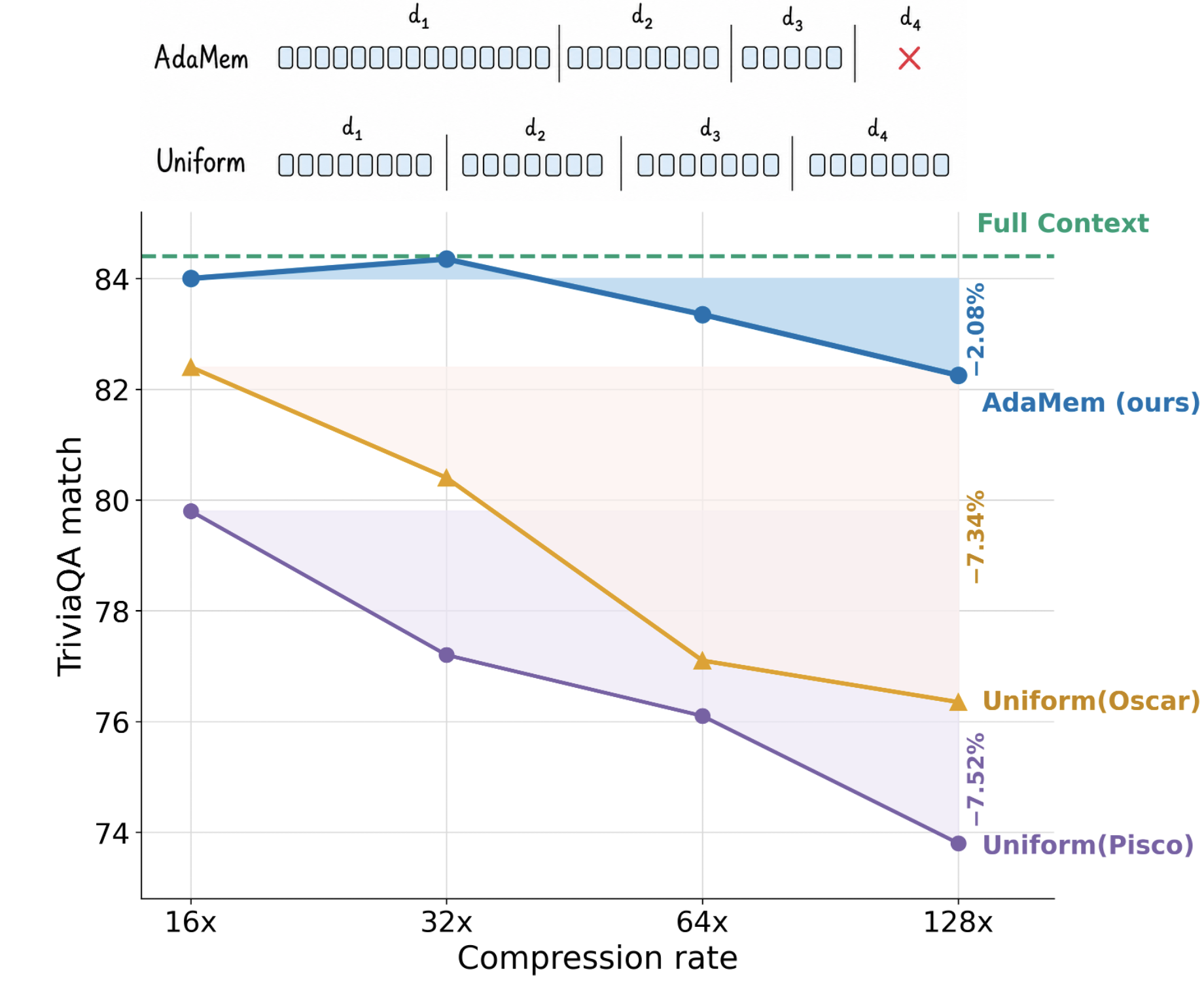}
    \caption{
        TriviaQA match as a function of the compression rate
        on 25 retrieved documents.
    }
    \label{fig:triviaqa-compression}
\end{figure}

A natural solution is to compress retrieved evidence before generation. \emph{Hard compression} shortens the text via sentence selection or token removal \cite{jiang2023llmlingua,xu2024recomp,pan2024llmlingua}, preserving interpretability but discarding information that may be collectively useful. \emph{Soft compression} maps each passage to a short sequence of continuous memory embeddings consumed directly by the decoder \cite{cheng2024xrag,chevalier2023adapting,ge2024icae}, achieving much higher compression rates. Recent methods such as OSCAR \cite{louis2026oscar} jointly produce passage representations and relevance scores in one pass, using scores only to order or discard passages, after which every retained document receives the same number of memory tokens. This binary allocation is suboptimal when passage utility is graded. Thus, the key question is not only which documents to retain, but how to distribute a fixed memory budget among documents with heterogeneous utility.


This setting also exposes an important limitation of prevailing evaluation protocols for soft-compressed RAG. Open-domain QA benchmarks typically provide question-answer supervision rather than exhaustive utility annotations for every document in a retrieved candidate pool~\citep{kwiatkowski-etal-2019-natural,joshi-etal-2017-triviaqa,mallen-etal-2023-trust}, while soft-compression models are commonly trained and evaluated on relatively small retrieved sets~\citep{cheng2024xrag,louis2025pisco,louis2026oscar}. Because small top-$k$ pools contain few candidates and exclude most lower-ranked retrievals, they expose models to fewer marginal and irrelevant passages and a narrower range of passage utility, making the benefits of reranking difficult to measure. In larger candidate pools, reranking must distinguish answer-bearing evidence from numerous marginal documents~\citep{nogueira2019passage,li-ouyang-2025-knowledge}, while the total representation budget cannot grow proportionally with the number of candidates. Evaluating under these conditions is therefore necessary to assess not only which documents should reach the generator, but also how the finite memory budget should be distributed among them.

Motivated by these limitations, we propose \textbf{\methodname}, a framework that transforms reranking from binary selection into relevance-aware memory allocation. A shared compressor produces query-conditioned memories and scores; a principled softmax allocation, grounded in logarithmic utility theory, distributes a fixed total memory budget across passages according to their estimated utility. The temperature interpolates between uniform allocation and hard top-\(k\) selection. Here are our main contributions:

\begin{itemize}
\item We introduce relevance-aware memory allocation for soft-compressed RAG: each document receives an adaptive memory budget determined by its estimated relevance (Figure~\ref{fig:triviaqa-compression}), concentrating capacity on answer-bearing evidence while preserving compact representations for complementary context. This adaptivity makes \methodname substantially more robust to extreme compression: between \(16\times\) and \(128\times\) its quality degrades by only about 2\%, compared with over 7\% for uniform-allocation baselines such as OSCAR.

\item Across six open-domain QA benchmarks, \methodname outperforms every evaluated compression baseline at matched budgets: at $16\times$ it improves over uniform allocation by $3.4\%$ on average (up to $5.5\%$) and over the strongest hard-compression baseline by $9.7\%$ despite an $8\times$ smaller context. Quality is also more stable under inference-time compression: moving from $16\times$ to $64\times$ costs \methodname $2.8\%$ on average, against $12.2\%$ for uniform allocation. At different compression rates, \methodname retains up to $99\%$ of full-context quality at $4\times$ lower latency.


\item We release \textbf{KILT-SCR} (Soft Compression and Reranking), a large-candidate-pool resource for reranking and context compression, with 25 passages per query, 196,916 training instances, and evaluation data from six QA benchmarks. The training split use answer-grounded filtration to ensure supporting retrieved evidence and gold-consistent generated targets.

\end{itemize}

\section{Related Work}
\label{sec:related_work}
Prior work on context compression follows two main paradigms: \emph{hard compression}, which shortens retrieved text, and \emph{soft compression}, which replaces it with continuous representations.

\paragraph{Hard Compression.}
\label{sec:hard_compression}

These methods retain the generator's textual interface by selecting or rewriting useful context. LLMLingua~\cite{jiang2023llmlingua} uses perplexity-based token removal; LongLLMLingua~\cite{jiang2024longllmlingua} adds query awareness; LLMLingua-2~\cite{pan2024llmlingua} formulates compression as token classification. RECOMP~\cite{xu2024recomp} learns extractive and abstractive compressors optimized for downstream performance. Such methods are interpretable and model-agnostic, but their discrete outputs must be re-encoded per query, limiting compression rates.

\paragraph{Soft Compression.}
\label{sec:soft_compression}

Soft compression maps text to continuous memory embeddings consumed directly by the generator. Gist tokens~\cite{mu2023gist}, Auto-Compressors~\cite{chevalier2023adapting}, and In-Context Autoencoders~\cite{ge2024icae} establish that LLMs can condition on latent representations. For RAG, xRAG~\cite{cheng2024xrag} projects a pre-computed embedding into generator space; 500xCompressor~\cite{li2025500x} demonstrates extreme compression; PISCO~\cite{louis2025pisco} trains a compressor via distillation. OSCAR~\cite{louis2026oscar} is most closely related: it performs online query-conditioned compression and jointly produces representations and relevance scores. However, as in most prior pipelines, relevance determines ordering or selection, while each retained passage receives the same fixed number of memory tokens. AdaMem instead distributes a fixed global budget non-uniformly according to estimated utility.

\paragraph{Capacity of Soft Memories.} 
\label{sec:capacity}
\citet{kuratov2025cramming} show a single vector can encode sequences at compression ratios approaching 1500$\times$, and that compressibility depends more on resolvable uncertainty than raw length; we build on this observation but study how a \emph{fixed} capacity should be \emph{distributed across multiple passages} rather than the ceiling of a single memory. 

\section{Method}

\begin{figure*}[t]
    \centering
    \includegraphics[width=\textwidth]{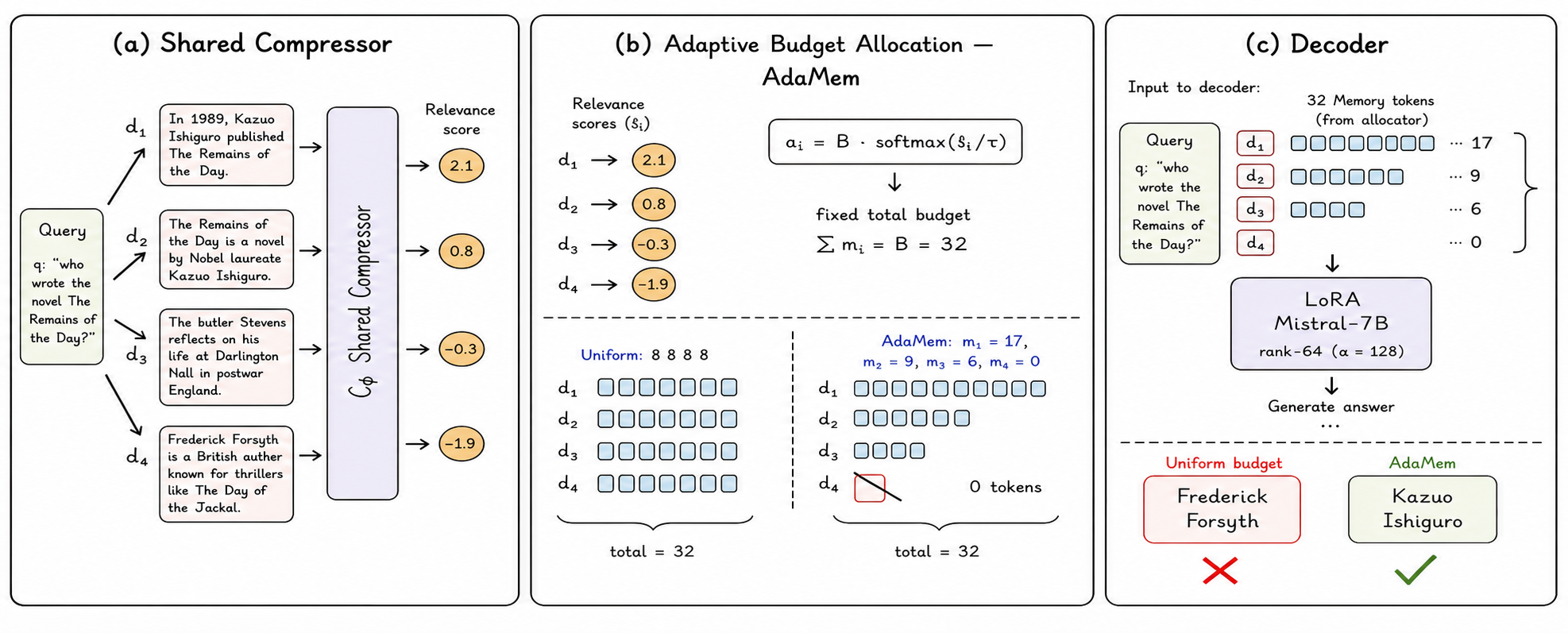}
\caption{
Overview of \methodname.
(a) A shared compressor jointly produces passage memory representations and relevance scores.
(b) The scores determine a query-dependent allocation of the fixed budget $B$ via softmax and largest-remainder rounding.
(c) The allocated memories are projected, ordered by relevance, and passed to a LoRA-adapted decoder for answer generation.
}
    \label{fig:method}
\end{figure*}

\paragraph{Problem Setup.}
We consider RAG in which a query $q$ is answered from
a pool of $K$ retrieved passages $\mathcal{D}=\{d_i\}_{i=1}^{K}$
(Figure~\ref{fig:method}a). Soft compression represents each passage $d_i$ by a short sequence of continuous \emph{memory embeddings} that
condition the decoder in its place. Writing $m_i$ for the number of memory
embeddings assigned to passage $i$, the total $B=\sum_{i=1}^{K} m_i$ is the
memory budget that governs the decoder's context length, and hence the cost of
generation. Existing soft-compression methods assign every retained passage the
same $m_i$, or set it by a length-based rule, spreading the budget without regard
to how useful each passage is for the query. We instead aim to distribute a
fixed budget $B$ according to query-specific utility, giving more memory
embeddings to answer-bearing passages, fewer to marginal ones, and none to
irrelevant ones, while keeping the total $\sum_i m_i=B$ unchanged. This poses two coupled subproblems:
estimating passage relevance, and turning those estimates into an integer
per-passage allocation of the budget. We next describe the shared compressor
that produces passage memories and relevance scores in a single pass, the rule
that converts these scores into a per-passage allocation, and the sense in
which that rule is optimal (Figure~\ref{fig:method}).

\paragraph{Query-Conditioned Compression and Adaptive Allocation.}
For each query--passage pair, we form
\begin{equation}
x_i=\texttt{Query: }q\texttt{\textbackslash nDocument: }d_i,
\end{equation}
truncate it to 128 tokens, and append a fixed maximum number of
\texttt{<MEM>} tokens and one \texttt{<RERANK>} token. A shared compressor
$C_{\phi}$ encodes $x_i$ once: the final-layer state of \texttt{<RERANK>}
yields a relevance score,
\begin{equation}
s_i=\mathbf{w}_{r}^{\top}C_{\phi}(x_i)_{\texttt{<RERANK>}}+b_r,
\end{equation}
whereas the states at the \texttt{<MEM>} positions form an ordered bank of
candidate memory embeddings for passage $d_i$ (Figure~\ref{fig:method}a).

Given the scores of the $K$ retrieved passages, we standardize them within the
pool as $\widehat{s}_i=(s_i-\bar{s})/(\sigma_s+\epsilon)$ and allocate a fixed
global memory budget $B$ according to
\begin{equation}
a_i=B\,\operatorname{softmax}(\widehat{s}_i/\tau),
\end{equation}
where $\tau>0$ controls allocation sharpness. We convert the fractional shares
$\{a_i\}_{i=1}^{K}$ into integer counts $\{m_i\}_{i=1}^{K}$ using
largest-remainder rounding, preserving $\sum_i m_i=B$; passages with $m_i=0$
are omitted \citep{Farber_2006} (Figure~\ref{fig:method}b). We set $B$ to the total number of embeddings
assigned by a uniform length-based reference, where passage $i$ would receive
$m_i^{(0)}=\lfloor\ell_i/\rho\rfloor+1$ embeddings for truncated length
$\ell_i$ and compression factor $\rho=16$.

For each passage, we retain its first $m_i$ \texttt{<MEM>} states, project
them into the decoder embedding space with a two-layer MLP, and concatenate
the resulting variable-length sequences in decreasing relevance order. The
decoder conditions on the query and these allocated memories to generate the
answer (Figure~\ref{fig:method}c). We optimize the compressor and projection
module in full, and adapt the decoder using rank-64 LoRA modules in all linear
layers ($\alpha=128$, dropout $0.1$); its token embeddings and
language-modeling head are also updated.

\begin{table*}[h!]
\centering
\small
\setlength{\tabcolsep}{4.5pt}
\renewcommand{\arraystretch}{1.0}
\begin{tabular}{lccccccc}
\toprule
Method & TriviaQA & HotpotQA & NQ & PopQA & ASQA & BioASQ & Compr.\ rate \\
\midrule
No Context   & 71.9 & 32.3 & 41.7 & 26.5 & 48.4 & 29.4 & -- \\
Full Context & 84.4 & 51.5 & 66.3 & 60.2 & 70.9 & 33.1 & -- \\
\midrule
RECOMP       & 76.0 & 37.0 & 50.0 & 34.8 & 55.9 & 29.0 & 200 \\
LLMLingua-2  & 80.7 & 43.3 & 60.9 & 44.3 & 63.0 & \hla \textbf{31.2} & 2 \\
LLMLingua    & 76.2 & 37.2 & 55.0 & 39.0 & 58.4 & 27.2 & 2 \\
\midrule
xRAG         & 26.9 & 13.9 & 9.3 & 11.2 & 11.9 & 8.8 & 147 \\
PISCO        & 79.8 & 38.4 & 58.9 & 53.6 & 61.0 & 27.7 & 16 \\
\midrule
\multirow{3}{*}{OSCAR}
             & 82.4 & 44.1 & \hlc 62.8 & 58.4 & 66.5 & 28.6 & 16 \\
             & 80.4 & 40.0 & 60.0 & 56.9 & 66.8 & 28.4 & 32 \\
             & 77.1 & 36.3 & 54.2 & 49.8 & 58.8 & 26.0 & 64 \\
\midrule
\multirow{3}{*}{\textbf{\methodname (proposed)}}
             & \hla \textbf{84.0} & \hla \textbf{46.3} & \hla \textbf{64.0} & \hlb 61.6 & \hla \textbf{69.1} & \hlc 29.2 & 16 \\
             & \hlb 83.9 & \hlb 45.4 & \hlb 63.9 & \hla \textbf{61.7} & \hlb 68.5 & \hlb 29.3 & 32 \\
             & \hlc 83.3 & \hlc 44.2 & 61.4 & \hlc 59.6 & \hlc 67.1 & 28.8 & 64 \\
\bottomrule
\end{tabular}
\caption{Match (\%) performance on six QA benchmarks with 25 retrieved documents. Hard- (RECOMP, LLMLingua,
LLMLingua-2) and soft-compression (xRAG, PISCO) baselines are shown at their
respective rates; OSCAR (uniform) and \methodname (relevance-aware) share one
compressor and are reported at $16$/$32$/$64\times$, with $32$/$64\times$ at
inference only. Per column, the top-3 compression methods are shaded green
(darkest\,=\,best). The bootstrap-estimated standard deviations for TriviaQA, HotpotQA, NQ, PopQA, ASQA, and BioASQ are 1.0, 1.1, 1.1, 1.1, 1.5, and 1.2, respectively. Generation temperature: $\tau=1$.}
\label{tab:main-results}
\end{table*}

\paragraph{Budget-Allocation View.}
The allocation rule is not a heuristic. If the value of a passage grows
logarithmically with its token count, the standard diminishing-returns model
of resource allocation~\citep{Kelly_1998}, then
$B\,\operatorname{softmax}(\widehat{s}_i/\tau)$ is the exact optimal split of
the budget for utilities $w_i=e^{\widehat{s}_i/\tau}$
(Proposition~\ref{prop:app-exact}, Appendix~\ref{app:allocation}). The
temperature interpolates between the two strategies of prior work:
$\tau\to\infty$ recovers the uniform split, $\tau\to0$ recovers
rerank-and-truncate, and the integer rounding step supplies passage dropping.
The classical exponential-distortion model of bit
allocation~\citep{Huang_1963,Shoham_1988} refines this picture with a
quantitative prediction: the softmax rule matches its reverse water-filling
optimum to first order precisely when
\begin{equation}
\tau=\frac{\gamma B}{K\sigma},
\end{equation}
where $\gamma$ is the distortion decay rate and $\sigma$ the spread of the
log-utilities (Appendix~\ref{app:allocation}).The optimal temperature should therefore decrease as the per-passage budget
$B/K$ becomes smaller. This trend is reflected in the experiments, where the
best $\tau$ decreases from $1.0$ at $16\times$ to $0.5$ at $64\times$
compression (Figure~\ref{fig:tau-sweep}).

\paragraph{Training Objective.}
The discrete allocation uses detached scores, we therefore supervise passage scoring directly by distilling cross-encoder teacher scores \(y_i\) via MSE:
\begin{equation}
\mathcal{L}_{\mathrm{rank}}=
\frac{1}{K}\sum_{i=1}^{K}(s_i-y_i)^2.
\end{equation}
The predicted scores are then used to allocate memory tokens and construct the
decoder context. For answer generation, we follow PISCO~\citep{louis2025pisco} and use
sequence-level knowledge distillation~\citep{kim-rush-2016-sequence}. Specifically, given a teacher-generated target answer
$a^{\mathrm{raw}}$, we optimize the autoregressive generation loss
$\mathcal{L}_{\mathrm{gen}}$ while masking the prompt prefix. The final
objective is
\begin{equation}
\mathcal{L}=
\mathcal{L}_{\mathrm{gen}}+
0.1\,\mathcal{L}_{\mathrm{rank}}.
\end{equation}
Thus, the shared compressor receives both generation and relevance-supervision
gradients, whereas the scoring head is calibrated explicitly by teacher-score
distillation.

\paragraph{Three-Stage Training.}
We train \methodname in three stages. \textit{Stage 1} pretrains the compressor and projection module with query-independent autoencoding and text-continuation objectives, teaching it to encode text into decoder-consumable representations. \textit{Stage 2} performs joint query-dependent training on an intermediate retrieval-augmented corpus, generating teacher answers from query-conditioned passage memories while initializing the reranker head via teacher-score distillation, adapting the compressor from generic to query-aware evidence representation. \textit{Stage 3} fine-tunes the full model on the filtered KILT-SCR corpus with full retrieved candidate pools, jointly optimizing generation and reranker distillation so that relevance scores directly determine passage ordering and memory allocation under the fixed budget at the target retrieval scale.


\begin{figure*}[ht!]
    \centering
    \includegraphics[width=\textwidth, height=8cm]{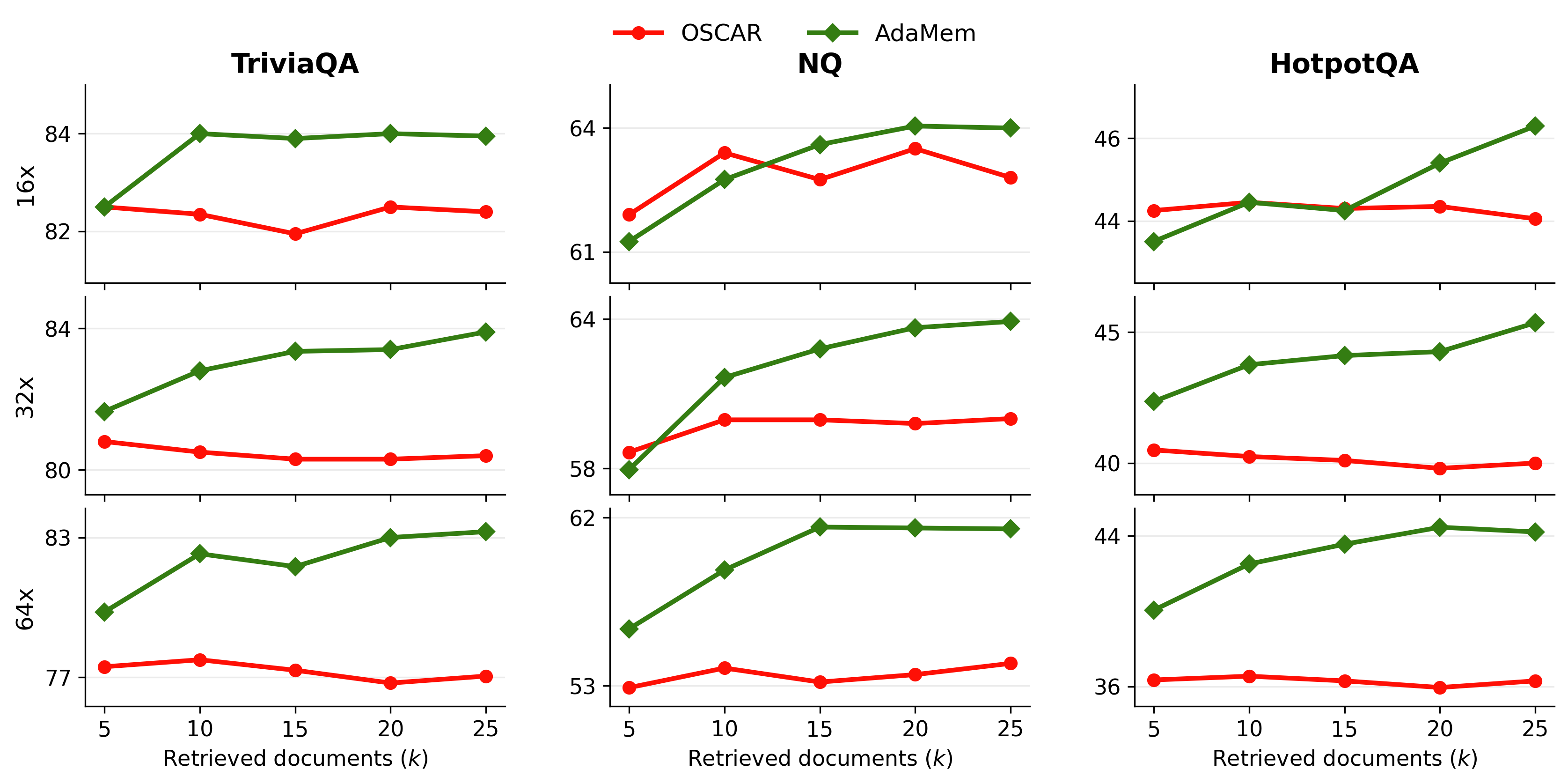}
    \caption{
        Match (\%) versus retrieval depth. AdaMem benefits from deeper retrieval more consistently than OSCAR, particularly under stronger compression.
    }
    \label{fig:oscar-ada-mem-delta}
\end{figure*}

\section{Experiments and Results}
\paragraph{KILT-SCR.} We construct KILT-SCR as a large-candidate-pool resource with training
and evaluation splits. Its training split is derived from the training sets of Natural Questions~\citep{kwiatkowski-etal-2019-natural}, HotpotQA~\citep{yang-etal-2018-hotpotqa}, and TriviaQA~\citep{joshi-etal-2017-triviaqa}. The evaluation split contains examples from the evaluation sets of these three datasets, together with PopQA~\citep{mallen-etal-2023-trust}, ASQA~\citep{stelmakh-etal-2022-asqa}, and BioASQ~\citep{krithara2023bioasq,tsatsaronis2015overview}. For every query, we retrieve a fixed pool of 25 passages from KILT-100w~\citep{petroni-etal-2021-kilt} using \texttt{BAAI/bge-large-en-v1.5}~\citep{xiao2024cpackpackedresourcesgeneral}.

For each training query, we score candidates with
\texttt{cross-encoder/ms-marco-MiniLM-L6-v2}~\citep{bajaj2018msmarcohumangenerated}
and provide the ten highest-scoring passages to
Mistral-7B-Instruct-v0.2~\citep{jiang2023mistral7b} to generate answer targets.
We retain only instances for which a gold-answer alias occurs both in at least
one retrieved passage and in the generated target, yielding 196,916 examples;
teacher scores are globally standardized over this split. In contrast, we apply
no answer-support or teacher-target filtering to KILT-SCR evaluation split, which contains
2,000 examples each from Natural Questions, HotpotQA, TriviaQA, and PopQA, 948
from ASQA, and 1,609 from BioASQ. At evaluation time, the learned reranker
scores the fixed candidate pool before answer generation.

\paragraph{Experimental Setup.}
\methodname uses a Llama-3.2-1B-Instruct~\citep{grattafiori2024llama3}
compressor and a Mistral-7B-Instruct-v0.2 decoder adapted with
LoRA~\citep{hu2022lora}. We first pretrain the compressor and projection module
on two million examples sampled from
\texttt{EleutherAI/SmolLM2-135M-10B}~\citep{allal2025smollm2smolgoesbig},
using autoencoding and language-modeling objectives. We then perform
query-dependent fine-tuning on \texttt{maxoul/pisco\_finetuning\_data}~\citep{louis2025pisco},
following the OSCAR training setup. Finally, we train on KILT-SCR with
the joint answer-generation and reranker-distillation objectives. For all experiments, we report Substring Match as the primary evaluation
metric. We provide additional LLM-as-a-judge metrics in the Supplementary
Material.

\paragraph{Baselines.}

We compare \methodname against hard-compression methods
(RECOMP, LLMLingua, and LLMLingua-2), and soft-compression methods
(xRAG, PISCO, and OSCAR). \textit{No Context} answers without retrieved
evidence, while \textit{Full Context} provides all retrieved passages without
compression. PISCO and OSCAR are the closest soft-compression baselines because they support
large retrieved candidate pools and integrate compression with the retrieval
pipeline. PISCO uses query-independent passage compression, whereas OSCAR, like
\methodname, performs query-conditioned compression and relevance scoring.
Both assign memory uniformly across retained passages, while \methodname
allocates a fixed global budget according to passage relevance. To isolate this
difference, we train PISCO, OSCAR, and \methodname using identical data splits,
hyperparameters, and training stages. Consequently, the comparison with OSCAR
directly measures the effect of relevance-aware memory allocation under a
matched setup.

\paragraph{Main Results.}
Table~\ref{tab:main-results} reports results with 25 retrieved documents. At the reference compression rate of 16×, \methodname achieves the best performance among compressed methods on five of six benchmarks. The only exception is BioASQ, where LLMLingua-2 achieves the strongest compressed result. Nevertheless, \methodname outperforms all evaluated soft-compression baselines.

Hard-compression methods exhibit a clear trade-off between compression and answer quality. RECOMP achieves the highest compression rate (200×), but its performance remains only moderately above the No Context baseline; for example, it improves HotpotQA from 32.3 to 37.0. Conversely, LLMLingua and LLMLingua-2 operate at only 2× compression, retaining a substantially larger text budget, yet generally remain below \methodname. For instance, at 16×, \methodname improves over LLMLingua-2 by 3.3 points on TriviaQA, 3.0 points on HotpotQA, and 17.3 points on PopQA. Thus, hard compression either requires aggressive token removal, which leaves limited useful evidence, or preserves much larger context without reaching the effectiveness of soft compression.

Among soft-compression methods, \methodname consistently outperforms OSCAR under the same 16× memory budget and yields substantial gains over PISCO. 
xRAG performs poorly in this setting because it is designed to compress a single document. When applied to long multi-document contexts, its compressed tokens become less informative and fail to preserve passage-specific, query-relevant evidence.

Both \methodname and OSCAR are trained at 16× compression. We evaluate higher compression ratios exclusively at inference time, without additional training or fine-tuning. At 32×, \methodname outperforms OSCAR at its training-time compression ratio of 16× across all benchmarks, indicating that it can use approximately half as many memory embeddings while maintaining answer quality. Even at 64×, \methodname surpasses OSCAR at 16× on five of six benchmarks. Moreover, \methodname retains a clear advantage over OSCAR at the same 64× compression rate. For example, on TriviaQA, increasing compression from 16× to 64× reduces performance by only 0.7 points for \methodname, compared with 5.3 points for OSCAR. This robustness suggests that relevance-aware allocation preserves salient evidence when the memory budget is severely constrained.

Finally, at 16×, \methodname closely approaches Full Context performance despite replacing the retrieved text with a substantially shorter continuous representation. On TriviaQA, it is within 0.4 points of Full Context, while on PopQA it exceeds Full Context by 1.4 points. Hence, selective memory allocation can both reduce computational cost and mitigate interference from irrelevant retrieved passages.

\begin{figure}[ht!]
    \centering
    \includegraphics[width=\columnwidth]{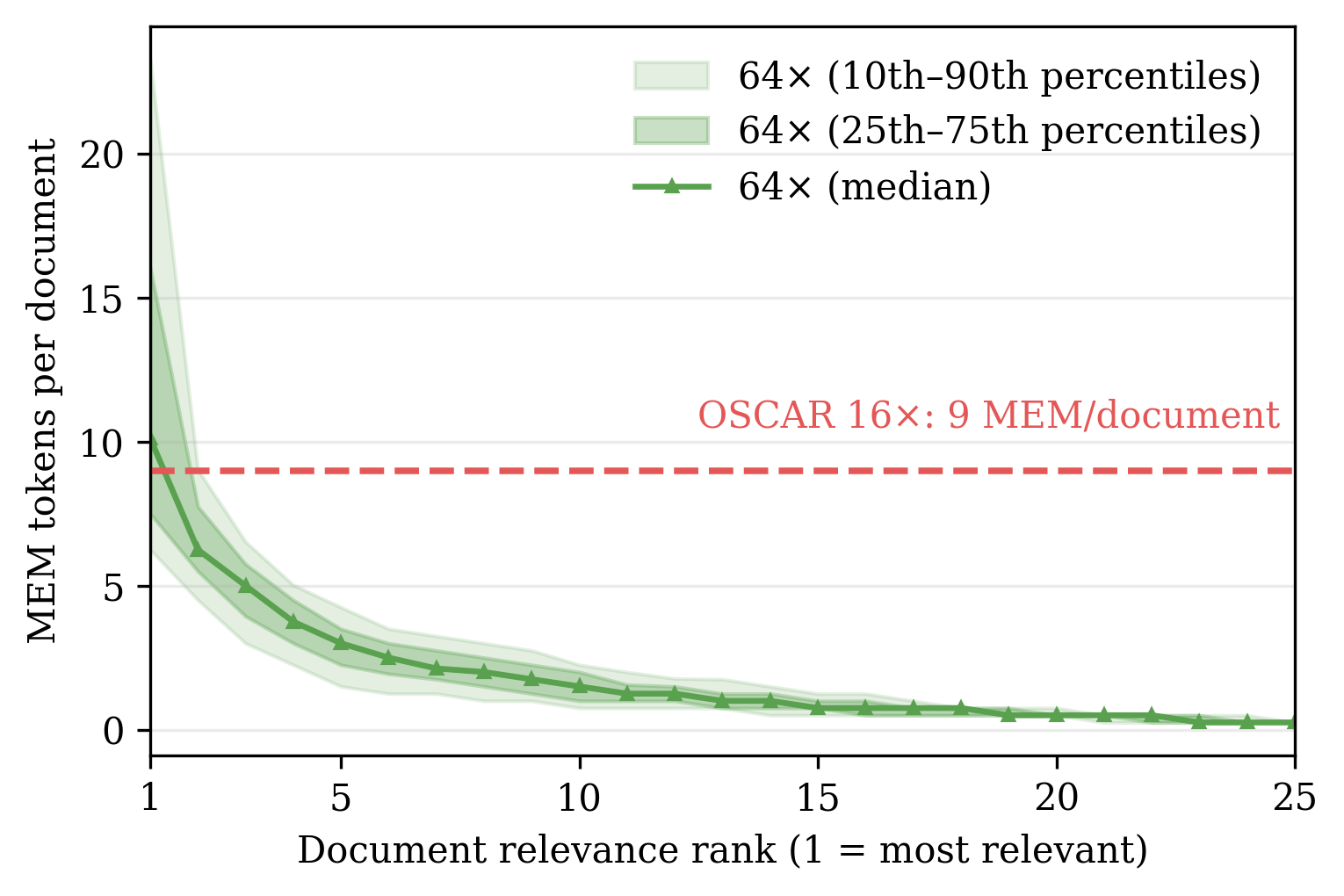}
    \caption{
    Memory-token allocation by document relevance rank at \(64\times\) compression.
    The dashed line denotes OSCAR's fixed \(16\times\) allocation.
    }
    \label{fig:mem-tokens-by-relevance}
\end{figure}

\paragraph{Memory Allocation Behavior.} Figure~\ref{fig:mem-tokens-by-relevance} illustrates the allocation mechanism
under this highly constrained \(64\times\) setting. For the most relevant
documents, \methodname allocates a median of roughly 10 memory tokens, already
exceeding OSCAR's fixed budget of 9 tokens per document at \(16\times\).
The 90th percentile is substantially higher, indicating that for some queries,
\methodname assigns considerably more capacity to the most relevant document.
In contrast, the allocation decreases sharply with relevance rank and
approaches zero for lower-ranked documents. Thus, \methodname concentrates
capacity on query-relevant evidence while assigning minimal capacity to noisy
passages, reducing their potential interference in the compressed context.

\paragraph{Robustness to Retrieval Depth and Compression.}

Figure~\ref{fig:oscar-ada-mem-delta} compares \methodname and OSCAR across retrieval depths \(k\in\{5,10,15,20,25\}\) and compression ratios of \(16\times\), \(32\times\), and \(64\times\) at $\tau$=1. \methodname benefits more consistently from deeper retrieval, whereas OSCAR remains nearly constant as additional lower-ranked documents are introduced.

This difference is most pronounced under aggressive compression. At \(64\times\), increasing \(k\) from 5 to 25 improves \methodname by 3.5 points on TriviaQA, 5.4 on Natural Questions, and 4.2 on HotpotQA, while OSCAR changes only marginally over the same range. A similar pattern holds at \(32\times\), particularly on Natural Questions and HotpotQA.

These results suggest that \methodname can extract useful evidence from less-relevant documents without allocating them the same budget as highly relevant ones. In contrast, OSCAR's fixed per-document allocation provides limited benefit from deeper retrieval, as its representation capacity is uniformly spread over an increasingly noisy document set. The effect is weaker at \(16\times\), where both methods have sufficient capacity for shallow retrieval pools, but becomes increasingly clear as the memory budget is reduced.

\begin{figure}[h!]
    \centering
    \includegraphics[width=\columnwidth]{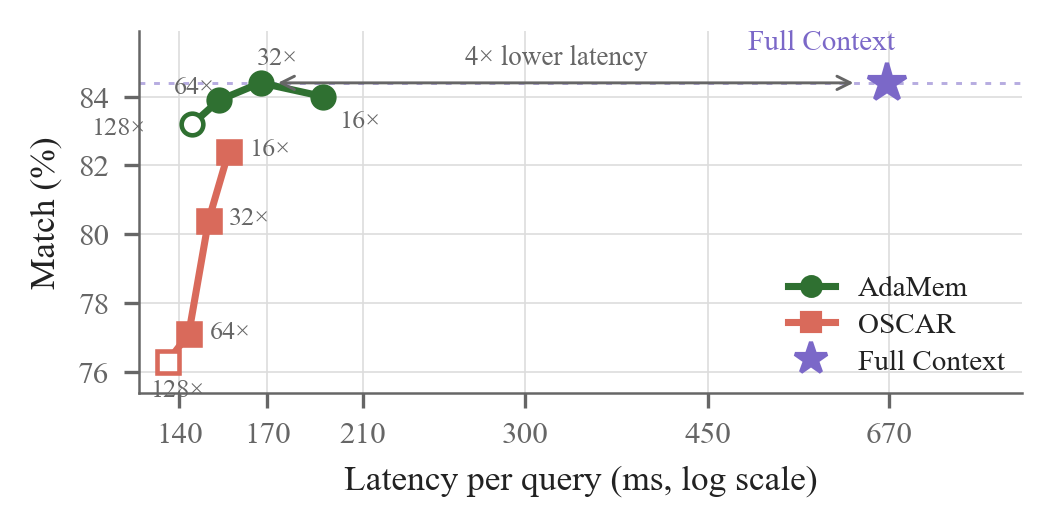}
    \caption{
    Quality-latency tradeoff on TriviaQA. \methodname shows similar latency as OSCAR and matches Full Context quality at $4.0\times$ lower latency.
    }
    \label{fig:latency-tradeoff}
\end{figure}

\paragraph{Efficiency-Performance Tradeoff.}
\label{sec:efficiency}

Figure~\ref{fig:latency-tradeoff} relates TriviaQA Match to per-query
inference latency across compression rates (theoretical optimal $\tau$ is used) .  \methodname
operates at the same low latency as prior soft compression while
delivering uncompressed-quality answers roughly \(4\times\) more efficiently than Full Context: at \(32\times\) compression it matches Full Context
quality (84.4 Match) at \(4.0\times\) lower latency and \(3.6\times\) less
compute. The two methods also trade quality for speed in sharply different
ways. Lowering the budget moves OSCAR down a steep curve, losing 5.3 Match points between \(16\times\) and \(64\times\) compression, whereas \methodname stays within 0.5 points over the same range: relevance-aware allocation
converts compression into latency savings almost without quality cost. As a result, \methodname at \(64\times\) is both faster and more accurate than uniform compression at \(16\times\).
Appendix~\ref{app:efficiency} provides additional efficiency discussion with detailed measurements, including throughput, compute, and memory information.

\paragraph{Effect of Training-Data Filtering.}

We evaluate the impact of filtered fine-tuning data. KILT-SCR retains instances where the retrieved pool contains a gold-answer alias and the teacher target is answer-consistent, reducing noise in both retrieval supervision and generation targets.

For a controlled comparison, we train the same \methodname model for one epoch on either filtered or unfiltered data under identical settings. Using 25 passages and a compression rate of 16, Table~\ref{tab:filtered-data} shows that KILT-SCR improves performance by 1.7 points on average, indicating better alignment between relevance estimation and answer generation. 

\begin{table}[h!]
\centering
\small
\begin{tabular}{lccc}
\toprule
Training data & TriviaQA & NQ & HotpotQA \\
\midrule
Unfiltered & 80.3 & 62.8 & 42.6 \\
Filtered & \textbf{82.8} & \textbf{63.3} & \textbf{44.5} \\
\midrule
Difference & +2.5 & +0.6 & +2.0 \\
\bottomrule
\end{tabular}
\caption{Effect of filtered fine-tuning data on Match (\%) at compression rate 16. Both models are trained for one epoch.}
\label{tab:filtered-data}
\end{table}

\paragraph{Ablation of Training Stages.}
\label{sec:train-stages}

Table~\ref{tab:training-stages} ablates the contribution of the three training
stages. 
Removing Stage~1 causes the largest performance drop across all datasets,
particularly on NQ and HotpotQA. This result shows that directly optimizing
query-dependent generation and relevance-supervision objectives is insufficient
to learn effective continuous passage representations from scratch. 
Removing either Stage~2 or Stage~3 produces smaller but consistent
degradations. Without Stage~2, the model loses the intermediate adaptation from
generic soft compression to query-conditioned evidence representation. Without
Stage~3, it is not optimized for relevance estimation and memory allocation
over the full retrieved candidate pools used at inference time.

\begin{table}[h!]
\centering
\small
\begin{tabular}{lccc}
\toprule
Training configuration & TriviaQA & NQ & HotpotQA \\
\midrule
w/o Stage 1
& 71.30 & 44.45 & 29.35 \\
w/o Stage 2
& 81.35 & 61.65 & 43.10 \\
w/o Stage 3
& 81.05 & 59.30 & 42.70 \\
\midrule
Full three-stage training
& \textbf{82.80} & \textbf{63.30} & \textbf{44.50} \\
\bottomrule
\end{tabular}
\caption{Effect of the three-stage training procedure at \(16\times\)
compression with \(k=25\).}
\label{tab:training-stages}
\end{table}

Overall, the results show that each stage addresses a distinct requirement:
learning general-purpose soft memories, adapting them to query-dependent
evidence, and aligning relevance-aware allocation with the target retrieval
setting.

\paragraph{Effect of Allocation Temperature.}
\label{sec:tau-dependence}

The temperature $\tau$ controls allocation sharpness: larger values approach
uniform allocation, while smaller values concentrate memory on the most
relevant passages. Our allocation view predicts that $\tau$ should decrease
as compression becomes stronger, since a smaller memory budget must be focused
on fewer high-utility passages.

Figure~\ref{fig:tau-sweep} confirms this prediction on TriviaQA. The optimal
temperature decreases from $\tau^*=1.0$ at $16\times$ compression to $0.7$ at
$32\times$ and $0.5$ at $64\times$ and $128\times$. Moreover, a theoretically
predicted schedule, $\tau \propto B(\rho)$ anchored at $16\times$, consistently
outperforms a fixed $\tau=1$ and closely tracks the empirical optimum, despite
a small gap at higher compression.

This effect contributes to the robustness in Figure~\ref{fig:triviaqa-compression}, which uses the theoretical temperature schedule without any empirical tuning.

\begin{figure}[h!]
    \centering
    \includegraphics[width=\columnwidth]{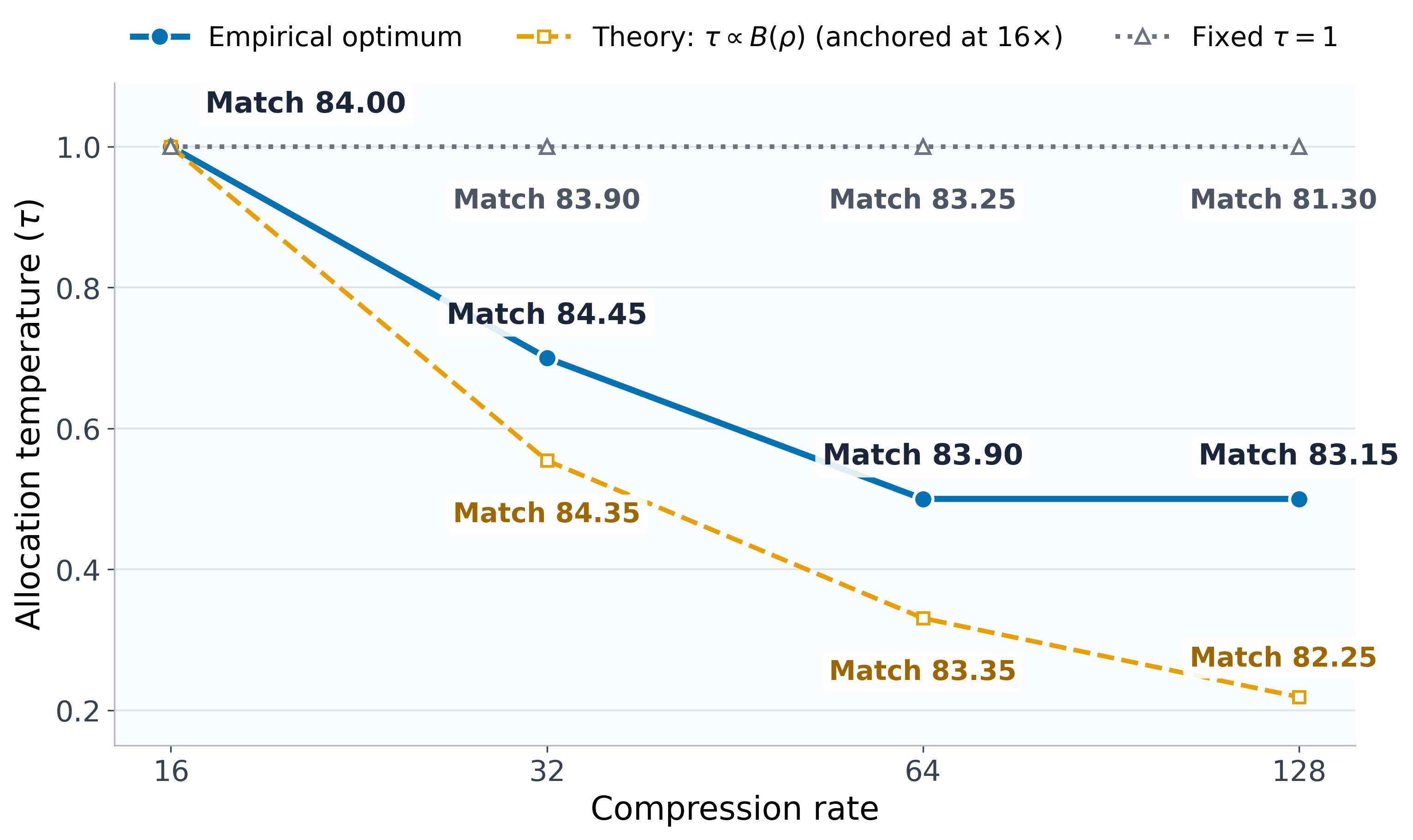}
    \caption{
        Optimal allocation temperature on TriviaQA across compression rates.
        The analytically predicted schedule ($\tau \propto B(\rho)$, anchored at
        $16\times$) follows the empirical optimum and consistently outperforms a
        fixed $\tau=1$.
        }
    \label{fig:tau-sweep}
\end{figure}

\section{Conclusion}

We introduced \methodname, a relevance-aware allocation framework for soft-compressed RAG. Instead of uniformly representing retained passages, \methodname converts query-dependent relevance scores into a non-uniform allocation of a fixed global memory budget, concentrating capacity on answer-bearing evidence while retaining limited capacity for complementary passages. 
The resulting allocation is the exact optimum of a proportional-fair
objective under an assumed logarithmic utility model, and is a first-order approximation to the reverse-water-filling optimum of classical exponential-distortion bit allocation.

\methodname establishes a new state of the art among soft-compression methods across six evaluated QA benchmarks. At \(16\times\) compression, it improves over baselines, the strongest prior soft-compression method, by an average of \(3.4\%\) in relative terms. The advantage becomes substantially larger as the memory budget decreases. At \(64\times\) compression, the average relative improvement reaches \(14.6\%\), with gains of up to \(21.8\%\) on HotpotQA. At \(32\times\), \methodname exceeds OSCAR at \(16\times\) on every benchmark, achieving an average relative improvement of \(2.9\%\). These results show that adaptive allocation is particularly effective under aggressive compression, where large retrieval pools must be represented within strict memory constraints.

Finally, we introduce KILT-SCR, a training and evaluation resource for the joint study of reranking and compression. Constructed to retain answer-bearing retrieved evidence and answer-consistent teacher targets, it enables controlled comparison of reranking and memory-allocation strategies over large candidate pools with heterogeneous passage utility.

Several future research directions remain open. The temperature could be learned or predicted per query rather than set per compression rate, allowing allocation sharpness to adapt to the confidence of the relevance estimates. The budget-allocation view itself is not specific to QA or memory embeddings and extends naturally to other bounded representations, such as KV-cache compression and long-context summarization. 

\bibliography{aaai2027}
\newpage
\appendix
\setcounter{secnumdepth}{1}
\section{Optimality of the Memory Allocation Rule}
\label{app:allocation}

This appendix states in what sense the allocation rule is optimal. The softmax
allocation is the exact optimizer of a logarithmic-utility budget-allocation
problem (Proposition~\ref{prop:app-exact}); a remark relates it to classical
bit allocation under exponential distortion, where it is instead a first-order
approximation; and the rounding step is the nearest budget-preserving
discretization (Proposition~\ref{prop:app-round}).

\paragraph{Setup.}
Let $K$ passages share a fixed budget of $B$ memory tokens, passage $i$
receiving $m_i$ tokens with $\sum_i m_i=B$, and let $w_i>0$ measure the utility
of passage $i$ for the query. An allocation model specifies how the value of a
passage depends on its token count; two standard choices are analyzed below.

\paragraph{Exact optimality under logarithmic utility.}
Suppose the value of a passage grows logarithmically with its token count, so
that the budget-optimal allocation solves
\begin{equation}
\max_{m\ge0}\ \sum_{i=1}^{K} w_i\ln m_i
\quad\text{subject to}\quad \sum_i m_i=B .
\label{eq:app-log}
\end{equation}
Logarithmic utility is the standard model of diminishing returns in resource
allocation; it defines the proportionally fair operating point
of~\citet{Kelly_1998} and coincides with the weighted Nash bargaining
objective. Soft memory tokens plausibly exhibit such diminishing returns: a
single token can already absorb a large but bounded amount of information
\citep{kuratov2025cramming}, so early tokens capture the most and later tokens
refine.

\begin{proposition}[Exact optimality of the softmax allocation]
\label{prop:app-exact}
For any scores $\widehat{s}\in\mathbb{R}^{K}$ and temperature $\tau>0$, the
allocation $a_i=B\,\operatorname{softmax}(\widehat{s}_i/\tau)$ is the unique
solution of \eqref{eq:app-log} with utilities $w_i=e^{\widehat{s}_i/\tau}$.
\end{proposition}
\begin{proof}
The objective is strictly concave and tends to $-\infty$ as any $m_i\to0$, so
the optimum is unique and interior and the nonnegativity constraints are
inactive. Stationarity of the Lagrangian gives $w_i/m_i=\lambda$ for every $i$,
hence $m_i=w_i/\lambda$; the budget constraint yields $\lambda=\sum_j w_j/B$,
so $m_i=B\,w_i/\sum_j w_j$. Substituting $w_i=e^{\widehat{s}_i/\tau}$ gives
$m_i=B\,\operatorname{softmax}(\widehat{s}_i/\tau)$.
\end{proof}

Three remarks follow. First, by the Gibbs variational principle,
$p=\operatorname{softmax}(\widehat{s}/\tau)$ is also the unique maximizer of
$\sum_i p_i\widehat{s}_i+\tau H(p)$ over the probability simplex, where $H$
denotes the Shannon entropy: the allocation maximizes the total relevance of
the allocated tokens under an entropy penalty on over-concentration, and
$\tau$ prices that penalty. Second, the temperature limits recover the two
prevailing strategies, uniform soft compression as $\tau\to\infty$ and hard
selection of the top-scoring passages as $\tau\to0$, so both sit at the
degenerate ends of the same family. Third, the exact optimum of
\eqref{eq:app-log} assigns every passage a strictly positive share, matching
the full support of the softmax; exact zeros arise only from the integer
constraint treated below. Within-pool standardization together with $\tau$
fixes the utility parametrization $w_i=e^{\widehat{s}_i/\tau}$; since the
relevance head is distilled to standardized cross-encoder scores, the
temperature selects how sharply the teacher's relevance scale is converted
into an allocation.

\paragraph{Exponential-distortion bit allocation.}
Classical bit allocation instead minimizes a weighted distortion
$\sum_i w_i D(m_i)$ under the same budget, with the high-resolution
exponential model $D(m)=c\,e^{-\gamma m}$,
$\gamma>0$~\citep{Huang_1963,Shoham_1988}. Ignoring nonnegativity, the equal marginal
return condition $\gamma w_i c\,e^{-\gamma m_i}=\lambda$ gives
\[
m_i^{\star}=\frac{B}{K}+\frac{1}{\gamma}\bigl(\ln w_i-\overline{\ln w}\bigr),
\qquad \overline{\ln w}=\tfrac1K\textstyle\sum_j\ln w_j ,
\]
which becomes a clipped affine function of centered log-utility once
$m_i\ge0$ is restored: shares that would be negative are set to zero and the
released budget is redistributed among the rest (\emph{reverse
water-filling}). Under this model the softmax rule is exact only to first
order: expanding $e^{\widehat{s}_i/\tau}=1+\widehat{s}_i/\tau+O(\tau^{-2})$
and using $\sum_j\widehat{s}_j=0$ gives
$B\,\operatorname{softmax}(\widehat{s}_i/\tau)
=B/K+(B/K\tau)\,\widehat{s}_i+O(\tau^{-2})$, which agrees with
$m_i^{\star}$ when $\widehat{s}_i=(\ln w_i-\overline{\ln w})/\sigma$ and
$\tau=\gamma B/(K\sigma)$, with $\sigma$ the spread of the log-utilities. The
two models thus assign the softmax rule different statuses: exact optimizer
under logarithmic utility, first-order approximation of reverse
water-filling under exponential distortion.

\paragraph{Integer rounding.}
The softmax yields fractional $a_i$ summing to $B$, whereas tokens are integers.
The natural discretization is the nearest budget-preserving integer vector,
\begin{equation}
\min_{m\in\mathbb{Z}_{\ge0}^{K}}\ \|m-a\|_2
\quad\text{subject to}\quad \sum_i m_i=B .
\label{eq:app-round}
\end{equation}

\begin{proposition}[Largest-remainder is nearest rounding]
\label{prop:app-round}
Largest-remainder rounding (floor every $a_i$, then hand the
$R=B-\sum_i\lfloor a_i\rfloor$ leftover tokens to the passages with the largest
fractional parts $r_i=a_i-\lfloor a_i\rfloor$) returns an exact minimizer of
\eqref{eq:app-round}.
\end{proposition}
\begin{proof}
If some minimizer had $m_i-a_i\ge1$, then, since the deviations sum to zero, some
$j$ has $m_j-a_j<0$; moving one token from $i$ to $j$ changes the squared error
by $2[(m_j-a_j)-(m_i-a_i)]+2<0$, a contradiction. Hence every optimal $m_i$ is
$\lfloor a_i\rfloor$ or $\lceil a_i\rceil$ (which also gives nonnegativity).
Writing $m_i=\lfloor a_i\rfloor+\delta_i$ with $\delta_i\in\{0,1\}$ and
$\sum_i\delta_i=R$, the objective becomes
$\sum_i r_i^2+\sum_i\delta_i(1-2r_i)$, minimized by placing the $R$ ones on the
largest remainders.
\end{proof}

\section{Inference Efficiency}
\label{app:efficiency}

\begin{table}[h!]
\centering
\footnotesize
\caption{Inference efficiency on TriviaQA with $k=25$: per-query latency,
throughput, total compute, and peak GPU memory across inference-time
compression rates.}
\label{tab:efficiency}
\begin{tabular}{@{}l@{\hspace{5pt}}r@{\hspace{5pt}}r@{\hspace{5pt}}r@{\hspace{5pt}}r@{\hspace{5pt}}r@{}}
\toprule
Method & Rate & Lat.\ (ms) & Thr.\ (q/s) & TFLOPs & Mem.\ (GB) \\
\midrule
Full Context & $1\times$ & 667.4 & 1.50 & 59.9 & 53.5 \\
\midrule
\multirow{4}{*}{OSCAR}
& $16\times$  & 156.1 & 6.41 & 15.0 & 39.6 \\
& $32\times$  & 149.3 & 6.70 & 13.4 & 39.0 \\
& $64\times$  & 142.9 & 7.00 & 12.5 & 38.7 \\
& $128\times$ & 136.4 & 7.33 & 12.1 & 38.5 \\
\midrule
\multirow{4}{*}{\methodname}
& $16\times$  & 192.4 & 5.20 & 21.0 & 56.4 \\
& $32\times$  & 167.5 & 5.97 & 16.6 & 48.4 \\
& $64\times$  & 152.9 & 6.54 & 14.3 & 44.4 \\
& $128\times$ & 144.0 & 6.95 & 13.0 & 42.3 \\
\bottomrule
\end{tabular}
\end{table}

Table~\ref{tab:efficiency} reports per-query latency, throughput, total
compute, and peak GPU memory on TriviaQA with $k=25$ retrieved passages, for
all inference-time compression rates used in the main text.

All compressed configurations are substantially cheaper than Full Context:
latency drops by \(3.5\times\) to \(4.9\times\) and compute by \(2.9\times\)
to \(5.0\times\), confirming that soft compression removes the dominant cost
of processing large retrieval pools. Relative to OSCAR at the same
compression rate, \methodname adds the cost of relevance scoring and of
encoding the fixed candidate memory bank before selection. This overhead is
confined to the compressor and does not grow with the decoder context, so its
share shrinks as compression increases: from 36.3 ms (23\%) at \(16\times\)
to 10.0 ms (7\%) at \(64\times\) and 7.6 ms (6\%) at \(128\times\). Because
the quality of \methodname is nearly flat in the compression rate, it can be operated at aggressive compression
where this overhead is smallest: at \(64\times\) and \(128\times\),
\methodname matches or improves on OSCAR at \(16\times\) in latency,
throughput, and compute while providing higher answer quality. Peak memory
remains moderately higher than OSCAR because the candidate bank for the full
pool is encoded before allocation, yet at \(64\times\) it is already below
the Full Context footprint (44.4 versus 53.5 GB), and the gap narrows further
with compression.

\section{Hardware Settings}
\label{app:hard}

All experiments were conducted on a single node with eight NVIDIA A100-SXM4 GPUs (80GB each). We used \texttt{bfloat16}, FlashAttention-2, and PyTorch distributed training via Accelerate and DeepSpeed with data parallelism. Pre-training used four GPUs with an effective batch size of 1,024, while both fine-tuning stages used eight GPUs with effective batch sizes of 512. Evaluation was distributed across eight GPUs, whereas latency and throughput were measured on a single A100 GPU.

\section{Judge Metrics}
\label{app:judge}

\begin{figure*}[t]
    \centering
    \includegraphics[width=\textwidth]{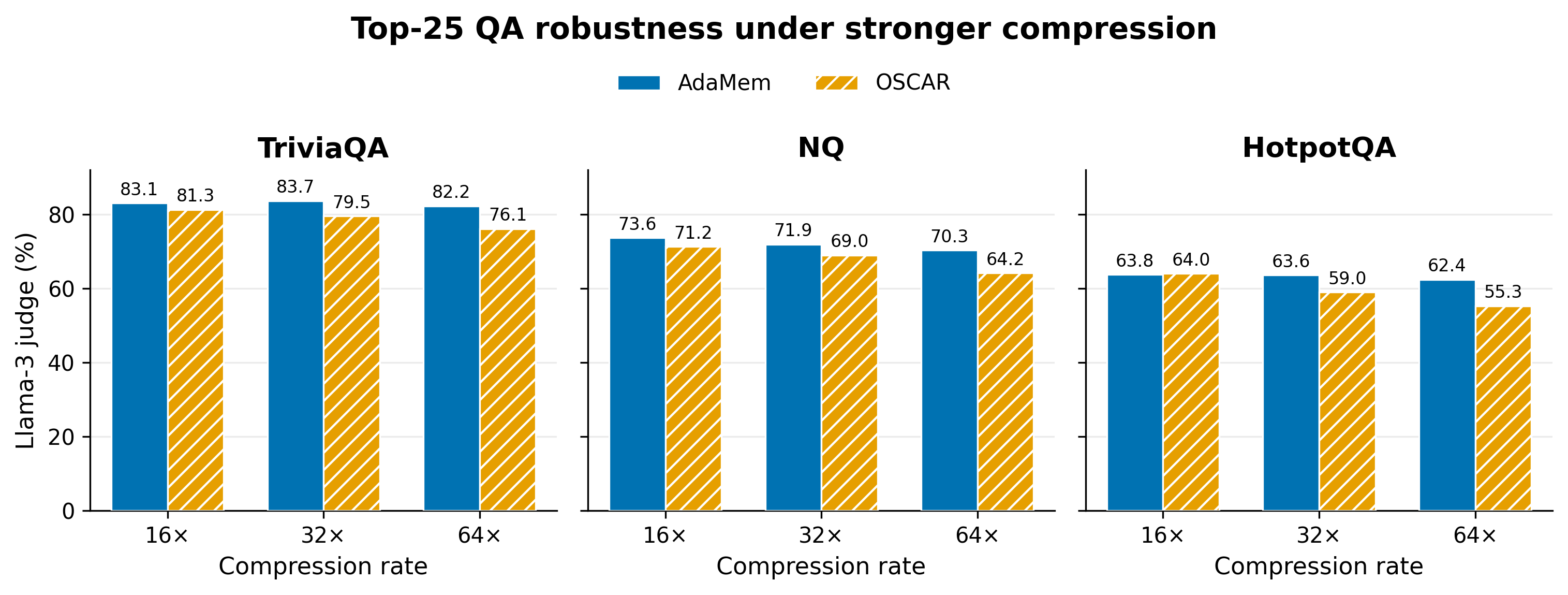}
    \caption{
    LLM-as-a-judge accuracy under increasingly strong inference-time compression.
    AdaMem consistently retains higher answer quality than OSCAR as the memory budget decreases.
    }
    \label{fig:judge-robustness}
\end{figure*}

We use Substring Match as the primary metric because it is deterministic and consistent with standard open-domain QA evaluation. However, it can overestimate answer quality when a model lists multiple candidate answers: the prediction receives full credit whenever it contains a gold-answer alias as a substring, even if the response is ambiguous, unsupported, or includes incorrect alternatives. We therefore additionally report LLM-as-a-judge scores, which evaluate the correctness of the generated answer in the context of the question and reference answer. The judge assigns a score of 1 for a correct answer, 0.5 for a partially correct answer, and 0 for a wrong answer; we report the mean score as a percentage. This complementary evaluation penalizes answer enumeration and verifies that improvements in Match correspond to genuinely correct answers rather than increased lexical overlap with the reference.

For each example, we provide the judge with the question, the reference answer, and the model prediction. We use the following prompt:

\begin{quote}
\small
\textbf{System prompt:} You are an evaluation tool. Answer with one of: 1: Correct, 0.5: Partially correct, 0: Wrong.

\textbf{User prompt:} Here is a question, a golden answer, and an AI-generated answer. Can you judge whether the AI-generated answer is correct according to the question and golden answer? Simply answer with one of: 1: Correct, 0.5: Partially correct, 0: Wrong.

Question: \texttt{\{question\}}\\
Golden answer: \texttt{\{answer\}}\\
Generated answer: \texttt{\{prediction\}}
\end{quote}

Figure~\ref{fig:judge-robustness} corroborates the Match-based robustness trends. At the training-time compression rate of $16\times$, AdaMem matches or exceeds OSCAR across all three benchmarks, improving by 1.8 points on TriviaQA and 2.4 points on NQ, while the methods are effectively tied on HotpotQA. The advantage grows under stronger inference-time compression. At $64\times$, AdaMem exceeds OSCAR by 6.1 points on TriviaQA and NQ, and by 7.1 points on HotpotQA. In contrast to OSCAR, AdaMem remains stable as the memory budget shrinks: its judge score decreases by only 0.9, 3.3, and 1.4 points from $16\times$ to $64\times$ on TriviaQA, NQ, and HotpotQA, respectively, compared with drops of 5.2, 7.0, and 8.7 points for OSCAR. These results show that relevance-aware allocation preserves answer-critical information under severe compression, rather than merely improving lexical overlap with reference answers.

\section{Training Schedule.}
Table~\ref{tab:training-schedule} summarizes the number of training examples and epochs used at each stage. We train for one epoch in Stages~1--2 and for three epochs during the final KILT-SCR fine-tuning stage.

\begin{table}[t]
\centering
\small
\caption{Training schedule for the three-stage procedure.}
\label{tab:training-schedule}
\begin{tabular}{lrr}
\toprule
Stage & Examples & Epochs \\
\midrule
Stage 1 & 2,000,000 & 1 \\
Stage 2 & 500,000 & 1 \\
Stage 3 (KILT-SCR) & 196,916 & 3 \\
\bottomrule
\end{tabular}
\end{table}
\end{document}